%% file: main.tex
\documentclass[conference]{IEEEtran}
\usepackage{hyperref}       
\usepackage{wrapfig}
\usepackage{subfig}
\usepackage{caption}
\usepackage{booktabs}       

\usepackage{cite}

\ifCLASSINFOpdf
  \usepackage[pdftex]{graphicx}
\else
\fi
\usepackage{amsmath,amssymb,amsthm}
\input{math_commands}

\newtheorem{theorem}{Theorem}

\usepackage{algorithm}
\usepackage{algorithmic}

\usepackage{array}
\usepackage{xurl}

\newcommand{\orcid}[1]{#1}
\begin{document}
%
\title{No learning rates needed: Introducing SALSA \\ - Stable Armijo
Line Search Adaptation}


\author{\IEEEauthorblockN{1\textsuperscript{st} Philip Kenneweg}
\IEEEauthorblockA{\textit{AG Machine Learning} \\
\textit{Bielefeld University}\\
Bielefeld, Germany \\
 \orcid{0000-0002-7097-173X}}
 \and
\IEEEauthorblockN{2\textsuperscript{nd} Tristan Kenneweg}
\IEEEauthorblockA{\textit{AG Machine Learning} \\
\textit{Bielefeld University}\\
Bielefeld, Germany \\
 \orcid{0000-0001-8213-9396}}
\and
\IEEEauthorblockN{3\textsuperscript{rd} Fabian Fumagalli}
\IEEEauthorblockA{\textit{AG Machine Learning} \\
\textit{Bielefeld University}\\
Bielefeld, Germany \\
\orcid{0000-0003-3955-3510}
\and
\IEEEauthorblockN{4\textsuperscript{th} Barbara Hammer}
\IEEEauthorblockA{\textit{AG Machine Learning} \\
\textit{Bielefeld University}\\
Bielefeld, Germany \\
 \orcid{0000-0002-0935-5591}}}}


\maketitle

\begin{abstract}
In recent studies, line search methods have been demonstrated to significantly enhance the performance of conventional stochastic gradient descent techniques across various datasets and architectures, while making an otherwise critical choice of learning rate schedule superfluous. 
In this paper, we identify problems of current state-of-the-art of line search methods, 
propose enhancements, and rigorously assess their effectiveness. Furthermore, we evaluate these methods on orders of magnitude larger datasets and more complex data domains than previously done.

More specifically, we enhance the Armijo line search method by speeding up its computation and incorporating a momentum term into the Armijo criterion, making it better suited for stochastic mini-batching. Our optimization approach outperforms both the previous Armijo implementation and a tuned learning rate schedule for the Adam and SGD optimizers.
Our evaluation covers a diverse range of architectures, such as Transformers, CNNs, and MLPs, as well as data domains, including NLP and image data.

Our work is publicly available as a Python package, which provides a simple Pytorch optimizer.
\end{abstract}


%
\IEEEpeerreviewmaketitle

\section{Introduction}
\label{sec:intro}
\input{intro}

\section{Background}
\label{sec:Background}
\input{background}

\section{Methods}
\label{sec:methods}
\input{methods}

\section{Experimental Approach}
\label{sec:experiments}
\input{experiments}

\section{Experimental Results}
\label{sec:results}

\input{results}


\section{Related Work}
\label{sec:related_work}

\input{related_work}

\section{Conclusion}
\label{sec:conclusion}
\input{conclusion}


\section*{Acknowledgment}

We gratefully
acknowledge funding and support by the Deutsche Forschungsgemeinschaft (DFG, German Research Foundation):
TRR 318/1 2021 – 438445824 and by SAIL. SAIL is funded by the Ministry of Culture and Science of the State of North Rhine-Westphalia under the grant no NW21-059A.



\bibliography{references}
\bibliographystyle{IEEEtran}
%

\end{document}

%% file: math_commands.tex
\usepackage{amsmath,amsfonts,bm}

\def\eqref#1{equation~\ref{#1}}

\def\1{\bm{1}}

\DeclareMathAlphabet{\mathsfit}{\encodingdefault}{\sfdefault}{m}{sl}
\SetMathAlphabet{\mathsfit}{bold}{\encodingdefault}{\sfdefault}{bx}{n}



%% file: intro.tex

In the field of modern machine learning, there are numerous optimization algorithms available \cite{schmidt2021descending}. However, determining the most suitable algorithm for a specific problem and finding the appropriate learning rate or learning rate schedule often requires extensive expertise and computational resources. In particular, the prevailing approach involves treating the learning rate as a hyperparameter and training the network repeatedly until the optimal value that yields the best performance is discovered. 
To simplify and expedite this process, recent research in deep learning \cite{vaswani20a, mahsereci15a, bollapragada18a, paquette20a} has proposed the reintroduction of line search methods as popular optimization technology, which effectively identify an appropriate learning rate by calculating the loss using different step sizes and comparing its reduction to the gradient of the loss function, hereby they eliminate costly hyperparameter tuning.

Traditional line search methods necessitate several forward passes per gradient update, making it imperative to explore more efficient alternatives. In the work by Vaswani et al. \cite{vaswani20a}, a Stochastic Line Search (SLS) is integrated with an intelligent re-initialization of the step size, effectively mitigating the requirement of multiple forward passes at each step.
This approach was shown in \cite{vaswani20a, ijcnn2023} to improve a variety of optimization methods, such as Stochastic Gradient Descent (SGD) on tasks such as matrix factorization as well as image classification for small networks and datasets. In \cite{vaswani2021adaptive} the authors adapt this line search to preconditioned optimizers like Adam \cite{adam} further increasing its usability. 

In this paper we extend upon this work, by introducing a momentum term to the SLS, critically improving its performance and stability. Furthermore, we introduce a limitation on the frequency with which a line search is performed, greatly reducing the computation needed.
Additionally, we conduct extensive experiments to evaluate the performance of various optimization methods across different datasets domains and architecture options.
Our findings demonstrate that, our improved Stable Armijo Line Search Adaptation algorithm, called SaLSa, consistently outperforms the previously introduced SLS as well as tuned optimizers, with very little computational overhead (about 3\% compared to no line search). We observe the SaLSa optimizers have on average an 1.5\% advantage on accuracy and a 50\% lower average log loss at end of training. Additionally, the stability of the training is improved compared to previously introduced versions of SLS. 

To make our work easy to reproduce and use, we implement all methods as PyTorch optimizers.
The source code is open-source and free  software (MIT licensed) and available on \href{https://github.com/TheMody/No-learning-rates-needed-Introducing-SALSA-Stable-Armijo-Line-Search-Adaptation}{github}. 

%% file: background.tex

The stochastic Armijo line search described in \cite{vaswani20a} and \cite{ijcnn2023} is designed to set a step size for all network parameters $w_k$ at iteration $k$.  In this section, we closely follow \cite{ijcnn2023} to formalize a modification of the Armijo criterion to handle the ADAM \cite{adam} direction. This is based originally upon \cite{vaswani20a, vaswani2021adaptive}
More importantly, we introduce an improved Armijo criterion, which mitigates the effect of noise in the mini-batch setting by calculating an exponential moving average smoothing on both sides of the Armijo equation.

We adopt the following notation from \cite{ijcnn2023}:
The loss function is denoted by $f(w)$. $|| \cdot ||$ denotes the Euclidean norm and $\nabla f$ denotes the gradient of $f$. Given the iteration counter $k$, $f_k$ and $\nabla f_k$ denote the mini-batch loss and its mini-batch gradient. 
\subsection{Armijo Line Search} 
\label{sec:sgdarmijo}
The Armijo line search criterion is defined as: 
\begin{equation}
    f_{k}(w_k + \eta_k d_k) \leq f_{k}(w_k) - c \cdot \eta_k ||\nabla f_{k}(w_k)||^2
    \label{eq:armijo}
\end{equation}
where $d_k$ is the direction (e.g., $d_k=-\nabla f_{k}(w_k)$ in case of SGD), 
$c \in (0,1)$ is a constant (commonly fixed to be $0.1$ \cite{vaswani20a}). The step size $\eta_k$ which satisfies Condition \ref{eq:armijo} is practically obtained by employing a backtracking procedure, see the pseudo-code in Algorithm \ref{alg:sls}:

\begin{algorithm}
\caption{Armijo Line Search}\label{alg:sls}
\begin{algorithmic}[1]
    \STATE $\eta_k = \eta_{k-1} \cdot 2 ^{1/b}$
    \WHILE {not $f_{k}(w_k + \eta_k d_k) \leq f_{k}(w_k) - c \cdot \eta_k ||\nabla f_{k}(w_k)||^2$}  
        \STATE $\eta_k = \eta_k \cdot \delta$
    \ENDWHILE 
    \STATE $w_{k+1} = w_k + d_k \cdot \eta_k$
\end{algorithmic}
\end{algorithm}
To avoid a monotonically decreasing step size, $\eta_k$ is increased each step as can be seen in line 1, ($b = 500$ and $\delta = 0.9$ in practice).

\subsection{Including preconditioned Optimizers (Adam)}

\begin{figure}
    \includegraphics[width = 0.48\textwidth]{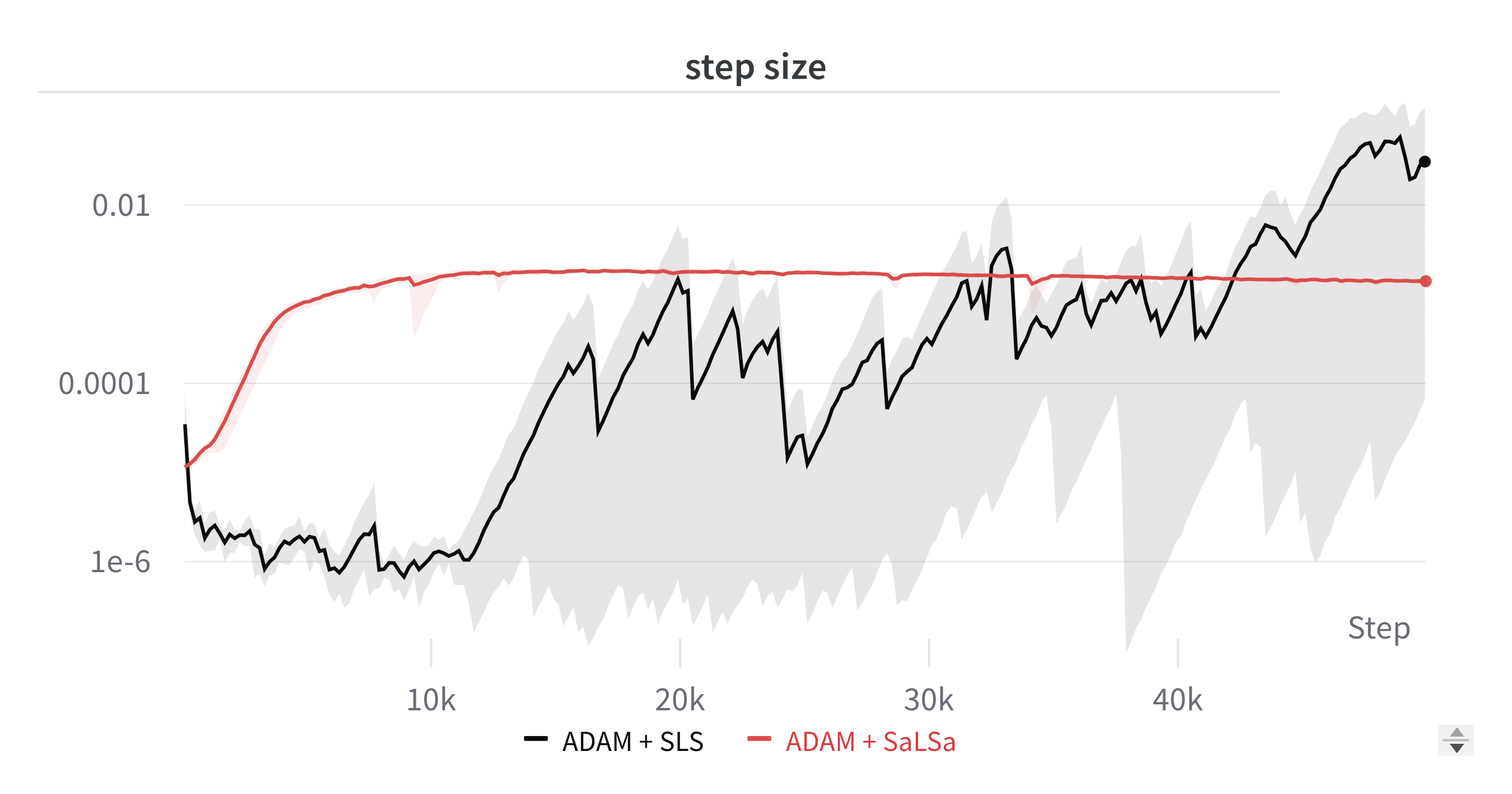}
  \caption{The step size of ADAM + SLS as well as ADAM + SaLSa on ImageNet. Colored areas indicate variance between runs. Notice the large variations for ADAM + SLS compared to the consistent and stable behavior of ADAM + SaLSa.}
\label{fig:ImageNet}
\end{figure}

\label{sec:adamarmijo}
In case of SGD, the direction $d_k$ is the negative mini-batch gradient.
\begin{equation*}
     d_k = -\nabla f_{k}(w_k)
\end{equation*}
Adam's direction defined in \cite{adam} can be written as:
\begin{equation}
\begin{split}
   g_{k} &= \nabla f_{ik}(w_k) \\
   m_{k} &= \beta_1 \cdot m_{k-1} + (1-\beta_1)\cdot g_{k} \\
   v_{k} &= \beta_2 \cdot v_{k-1} + (1-\beta_2)\cdot g_{k}^2 \\
   \hat m_{k} &=  m_{k-1} /(1-\beta^k_1) \\
   \hat v_{k} &=  v_{k-1} /(1-\beta^k_2) \\
    d_k &= -\hat m_{k} /(\sqrt{\hat v_{k} }+\epsilon)
\end{split}
\label{eq:adamopt}
\end{equation}

Adam integrates a momentum-based strategy along with a step-size correction based on gradient variance. In the context of training Transformers, these adjustments have proven to be significant improvements compared to the more straightforward SGD algorithm \cite{chen22a}.

The Armijo line search criterion from Eq. \ref{eq:armijo} must be adjusted for the Adam optimizer. We perform this adjustment following \cite{ijcnn2023} based on \cite{vaswani20a, vaswani2021adaptive}.
To check if the Armijo line search criterion is satisfied in the Adam case, we use the direction $d_k$ defined in Eq. \ref{eq:adamopt}, with momentum $\beta_1 = 0$. Note that, the Armijo criterion is only guaranteed to be satisfy-able by adjusting the step size $\eta_k$, if the update direction and the gradient direction are similar enough. However, this condition is not generally met when $\beta_1 \neq 0$ in Eq. \ref{eq:adamopt}. 
Additionally, we replace the gradient norm term $||\nabla f_{k}(w_k)||^2$ by the preconditioned gradient norm $\frac{||\nabla f_{k}(w_k)||^2}{\sqrt{\hat v_{k} }+\epsilon}$ as in \cite{vaswani2021adaptive} resulting in Eq. \ref{eq:armijoadam}.
\begin{equation}
    f_{k}(w_k + \eta_k d_k) \leq f_{k}(w_k) - c \cdot \eta_k \frac{||\nabla f_{k}(w_k)||^2}{\sqrt{\hat v_{k} }+\epsilon}
    \label{eq:armijoadam}
\end{equation}

Note that to perform final weight updates each step we use $\beta_1 \neq 0$.

\subsection{SLS Failure Cases}
\begin{figure}
    \includegraphics[width = 0.43\textwidth]{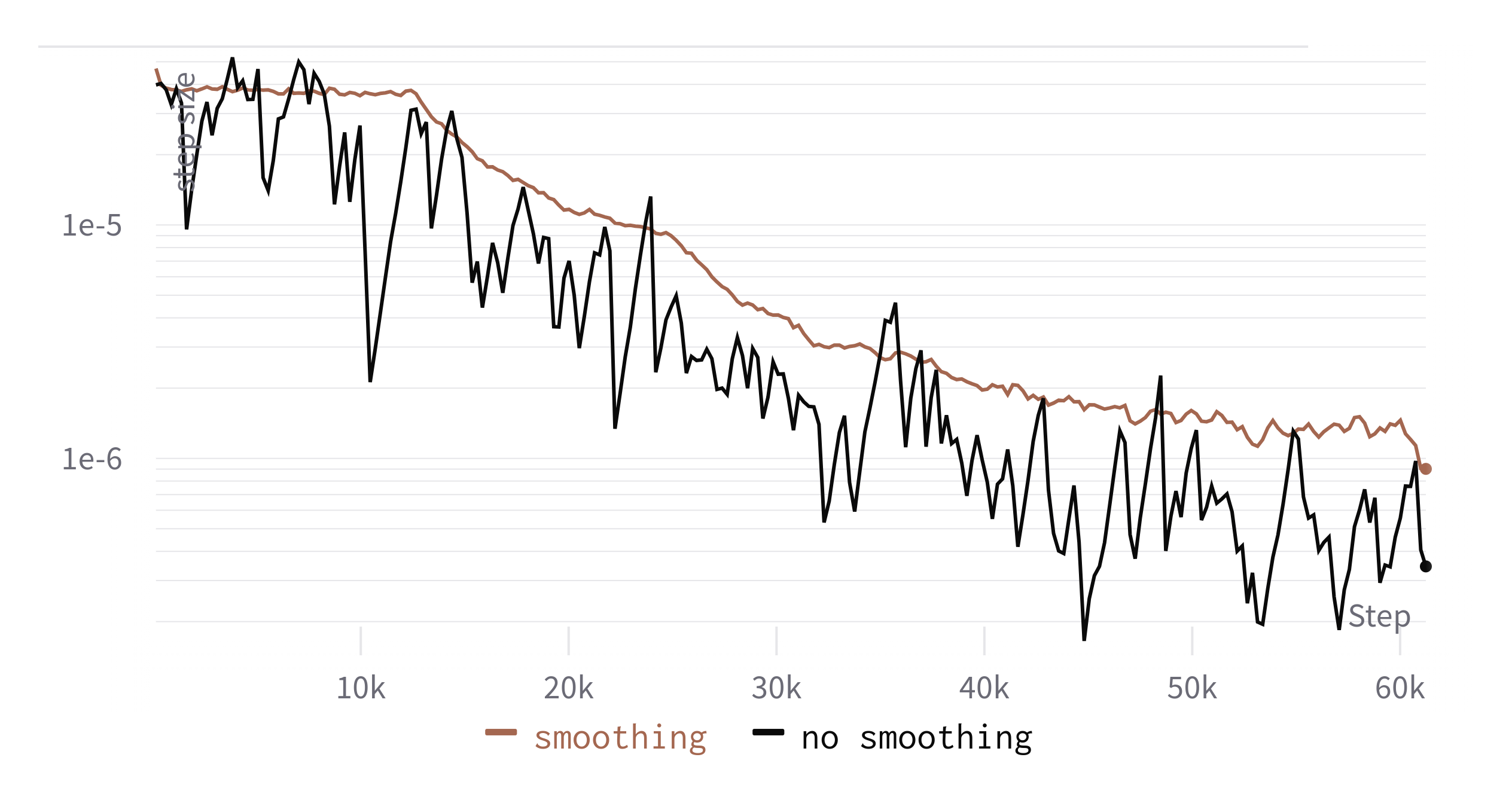}
  \caption{The step size of ADAM + SLS (black) compared to ADAM + SaLSa (brown) visualized during a training run of BERT on the MNLI dataset.}
\label{fig:smoothed}
\end{figure}
As shown in \cite{vaswani20a, vaswani2021adaptive} the previously described line search methods perform well on smaller datasets and neural network architectures. However, here we show that these methods have problems to consistently perform during larger scale training.


The first of these problems we call "mini-batch noise":
Eq. \ref{eq:armijo} and \ref{eq:armijoadam} describe criterions which are checked for every mini-batch.
This is problematic, since the criteria will be violated subject to inherent noise in the mini-batch data. 
The phenomenon is amplified by small mini-batch sizes. 
As can be seen in Figure \ref{fig:smoothed} in a typical training run the Armijo line search method leads to frequent changes of the step size, see Figure \ref{fig:smoothed}. 

\begin{figure}
    \includegraphics[width = 0.48\textwidth]{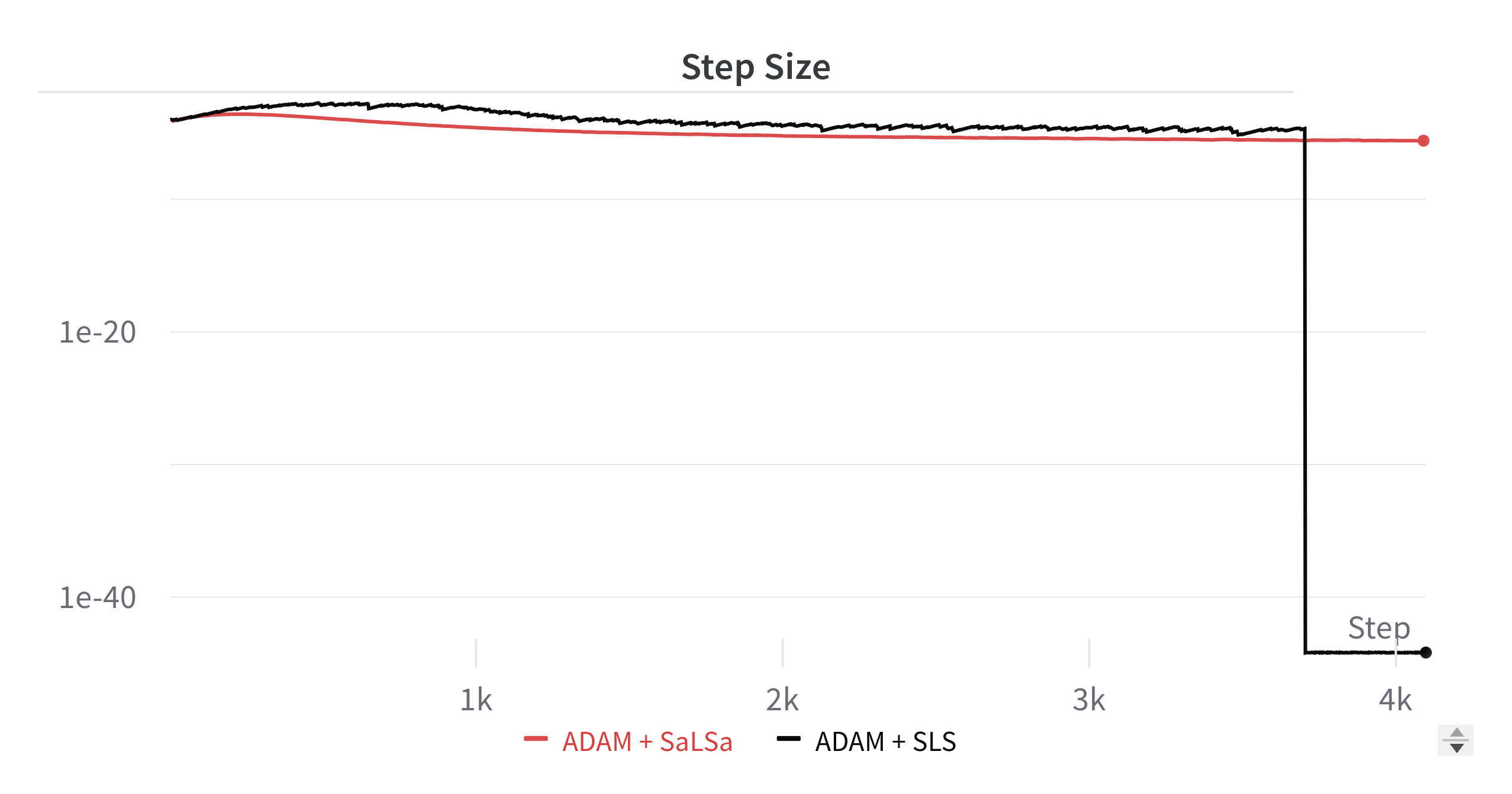}
  \caption{The step size of ADAM + SLS (black) on the CIFAR10 dataset. We sometimes observed drastic drops in step size due to computational precision problems. When training with SaLSa (red) we do not observe any such drops.}
\label{fig:gradprob}
\end{figure}

Another frequently occurring problem is mini-batch gradient approximations $\nabla f_{k}(w_k) \approx 0$. These are due to computational precision problems, even with float32 precision enabled. For an example see Figure \ref{fig:gradprob}. In the original implementation by \cite{vaswani2021adaptive}, whenever $\nabla f_{k}(w_k) \leq 10^{-8}$ no line search was performed.

Additionally, a problem which only occurs on large datasets, can be seen in Figure \ref{fig:ImageNet}. The step size and its variance over 5 different runs on ImageNet is visualized. We observe that the step sizes of SLS are very sensitive to initial conditions, the only difference between the runs is the random parameter initialization of the network and the shuffled dataset. This problem only seems to occur on larger datasets and is quite impactful, as some runs do not converge properly, or take very long to even begin converging.

%% file: methods.tex
To obtain a line search method with better properties in the mini-batch scenario, we propose to extend the Armijo criterion with a momentum term. Below we provide a detailed explanation of the modifications we made, the theoretical basis behind them, and our reasoning. Furthermore, we introduce a method to greatly reduce the computational overhead introduced by a line search.

\label{sec:mini-batch}
\subsection{Addressing Mini-batch Noise}


As an extension to Eq. \ref{eq:armijo} we propose an exponential smoothing (also called momentum term) of all factors which are dependent on the mini batch in this equation.
First we rewrite Eq. \ref{eq:armijo}:

\begin{equation}
    f_{k}(w_k) - f_{k}(w_k + \eta_k d_k) \geq  c \cdot \eta_k ||\nabla f_{k}(w_k)||^2
    \label{eq:armijonew}
\end{equation}

$f_{k}(w_k) - f_{k}(w_k + \eta_k d_k)$ denotes the decrease in loss and $||\nabla f_{k}(w_k)||^2$ denotes the gradient norm.
In order to apply exponential smoothing to both terms we define $h_{k}$ and $s_{k}$ as follows:
\begin{equation}
\begin{split}
    h_{k} = h_{k-1}\cdot\beta_3 + (f_{k}(w_k) - f_{k}(w_k + \eta_k d_k))\cdot(1-\beta_3) \\
    s_{k} = s_{k-1}\cdot\beta_3 + ||\nabla f_{k}(w_k)||^2\cdot(1-\beta_3)
\end{split}
\label{eq:smoothing}
\end{equation}
$h_{k}$ represents the smoothed decrease of the loss with the current step size, $s_{k}$ the smoothed gradient norm and $\beta_3 \in (0,1)$ the smoothing factor used for the exponential moving average.

We introduce the Stable Armijo Line Search Adaptation (SaLSa) criterion  as:

\begin{equation}
    h_{k} \geq  c \cdot \eta_k \cdot s_{k}
    \label{eq:armijosmoothed}
\end{equation}

Combining SaLSa and the Adam optimizer is done by computing $s_{k}$ as follows:

\begin{equation}
\begin{split}
    s_{k} = s_{k-1}\cdot\beta_3 + \frac{||\nabla f_{k}(w_k)||^2}{\sqrt{\hat v_{k}} +\epsilon} \cdot(1-\beta_3)
\end{split}
\label{eq:smoothingadam}
\end{equation}

and computing $d_k$ as described in Eq. \ref{eq:adamopt}, but with $\beta_1 = 0$. We keep the calculation procedure for the step size $\eta_k$ the same as previously described.

\subsection{Intuitive Motivation}
As mentioned in Section \ref{sec:mini-batch} due to the inherent noise in mini-batches we expect some of them to violate the original Armijo line search, even if the step size $\eta_k$ is appropriate for the majority of mini-batches around step $k$ in training.

Let us assume that all mini-batches are normally distributed with respect to the Armijo criterion, e.g. some mini-batches fulfill the condition with a wide margin, most fulfill it with a small margin and some rare exceptional batches violate the criterion. In this scenario we do not want to decrease the step size by a large amount each time we get an exceptionally bad mini-batch, since the step size is still fitting for most batches and the step size is only increased slowly. The stable Armijo line search adaptation in Equation \ref{eq:armijosmoothed} is implementing exactly this behaviour. Note that the exact distribution is not relevant in this thought experiment.

If we analyze Equation \ref{eq:armijosmoothed}, we notice that the right hand side is affected by the current step size $\eta_k$ to the same degree as in the original Armijo line search Equation \ref{eq:armijo}. However, the left hand side is substantially less affected since $\eta_k$ is part of the exponential smoothing process. This results in a slower reduction of the step size $\eta_k$ as the criterion is fulfilled more easily than the original criterion by lowering $\eta_k$.






\subsection{Theoretical Analysis}

We extend the convergence Theorem introduced in the original Armijo paper \cite{armijo66a} for the full batch setting and the SaLSa criterion with SGD from Eq. \ref{eq:armijosmoothed}.
We additionally require that every found learning rate yields an improved loss $f(w_k) - f(w_k + \eta_k d_k) \geq 0$.
This condition ensures the convergence for an infinite sequence, which may otherwise not be guaranteed due to the exponential smoothing.
In practice, enforcing this condition did not yield a significant difference in the optimization process and is thus not implemented in our experiments.

\begin{theorem}[Convergence Theorem]
\label{theorem1}
    Let $f \equiv f_k$.
    For $w_0 \in \mathbb{R}^d$ let $S(w_0) := \{w \mid f(w) \leq f(w_0)\}$ and assume that $f(w^*) := \inf_{w \in \mathbb{R}^d} f(w)$ exists for a unique point $w^* \in \mathbb{R}^d$ with $\nabla f(w) = 0$ for $w \in S(w_0)$ if and only if $w = w^*$.
    Any sequence $\{w_{k}\}_{k=1}^\infty$ found by the SaLSa criterion with $f(w_k) - f(w_k + \eta_k d_k) \geq 0$ and $c < 1$ converges to $w^*$.
\end{theorem}


\begin{proof}
    The condition $f(w_k) - f(w_k + \eta_k d_k) \geq 0$ ensures that $\{f(w_k)\}_{k=1}^K$ is non-increasing and thus any infinite sequence will converge to $f(w^*)$, given the assumptions.
    It thus remains to show, that in every step such a step size $\eta_k$ can be found.
    By definition, we have that
    \begin{align*}
        h_k &= \beta_3 h_{k-1} + (1-\beta_3) \left[f(w_k) - f(w_k + \eta_k d_k)\right]
        \\
        &\geq \beta_3 c \eta_{k-1} s_{k-1} + (1-\beta_3) \left[f(w_k) - f(w_k + \eta_k d_k)\right].
    \end{align*}
        If we assume, that there exists a learning rate $\eta_k\leq\eta_{k-1}$, such that $f(w_k) - f(w_k + \eta_k d_k) \geq c \eta_k \Vert \nabla f(w_k)\Vert^2$ (see proof of existence below), we can show that
    \begin{align*}
        \beta_3 c \eta_{k-1} s_{k-1} + (1-\beta_3) \left[f(w_k) - f(w_k + \eta_k d_k)\right] &\geq \\ c \eta_{k}\left[\beta_3 s_{k-1} + (1-\beta_3) \Vert \nabla f(w_k)\Vert^2\right]
        \\
= c \eta_{k} s_{k},
    \end{align*}
    which finishes the proof, as we have found a learning rate $\eta_k$ that fulfills the SaLSa criterion.
    
    We now prove the existence of $\eta_k\leq\eta_{k-1}$ with $f(w_k) - f(w_k + \eta_k d_k) \geq c \eta_k \Vert \nabla f(w_k)\Vert^2$ by contradiction, i.e. we assume that such a $\eta_k$ does not exist and thus $f(w_k) - f(w_k + \eta_k d_k) < c \eta_k \Vert \nabla f(w_k)\Vert^2$ for $\eta_k \leq \eta_{k-1}$. 
    Using the Taylor expansion for $f$ around $f(w_k)$ yields 
    \begin{align*}
         f(w_k) - f(w_k + \eta_k d_k) = -\eta_k d_k \nabla f (w_k) - o(\eta_k).
    \end{align*}
    For $\eta_k \leq \eta_{k-1}$ it follows then
    \begin{align*}
        c \eta_k \Vert \nabla f(w_k)\Vert^2 > -\eta_k d_k \nabla f (w_k) - o(\eta_k).
    \end{align*}
    Dividing both sides by $\eta_k$ and taking the limit for $\eta_k \to 0$ yields
    \begin{align*}
         c \Vert \nabla f(w_k)\Vert^2 > -d_k \nabla f (w_k),
    \end{align*}
    and thus with $d_k = - \nabla f(w_k)$ it follows $c > 1$, which is a contradiction.
\end{proof}

\subsection{Addressing Computational Costs}
\label{sec:speedup}
It is unnecessary and computationally expensive to perform a line search for every step during training, as for most steps the step size does not need to be changed. The overall training compute cost increases by roughly 30\% when performing a line search every step. To address this, we propose to perform a line search more regularly when a high rate of change of the step size is detected and less regularly otherwise. We realize this by keeping 2 different exponential moving averages of the step size $\eta_k$ which we update after every line search procedure:
\begin{equation}
     \Bar{\eta}_{k}(\beta) = \beta \Bar{\eta}_{k-1} + (1-\beta) \cdot \eta_{k-1}
\end{equation}
We calculate the average rate of change as follows:
\begin{equation}
     r_{k} = \frac{\Bar{\eta}_{k}(0.9)}{\Bar{\eta}_{k}(0.99)}
\end{equation}
and invert it if $r_{k} \leq 1$:
\begin{equation}
     \Bar{r}_{k} = \begin{cases} 
r_{k} & \text{if }r_{k} \geq 1 \\
r_{k}^{-1}  & \text{otherwise}
\end{cases}
\end{equation}

we set the line search frequency $L_{k}$ to the closest integer of:

\begin{equation}
     L_{k} = \frac{1}{\Bar{r}_{k}  -1}
\end{equation}

and clamp it to the range $L_{k+1} \in [1,10]$.
We perform the line search every $L_{k+1}$ steps. 
This reduces the extra compute needed from roughly 30\% to approximately 3\% for longer runs, see Figure \ref{fig:wallclock}. In practice, we did not notice any performance degradation.

\begin{figure}
    \includegraphics[width = 0.48\textwidth]{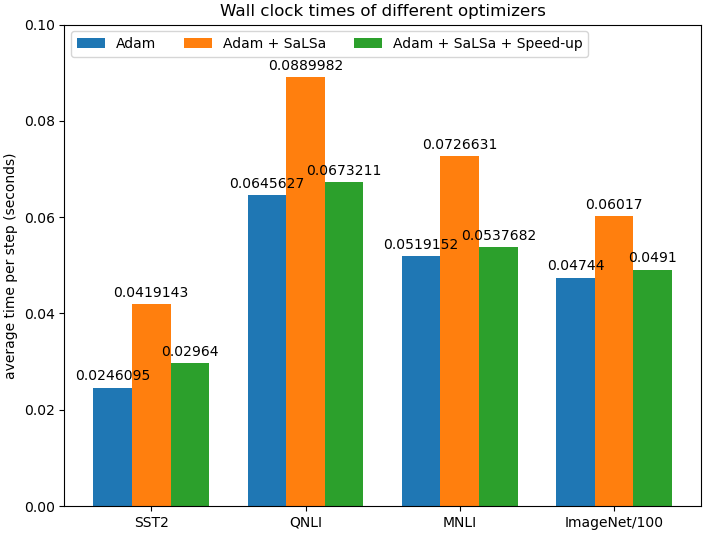}
  \caption{Wall clock run times on various datasets using the Speed-up described in \ref{sec:speedup}}
\label{fig:wallclock}
\end{figure}

\subsection{Practical Considerations}

In the original Armijo line search implementation a few outliers dominated the determination of $\eta_k$ as shown in Figure \ref{fig:smoothed}. The hyperparameter $c$ was set with this in mind. 
In our experiments we found good values for $c$ to be in the range $c \in [0.3,0.5]$.  
For all our experiments we used $c= 0.3$ (compared to $c=0.1$ for the original Armijo line search criterion). 

Furthermore, we tuned the hyperparameter $\beta_3 \in [0.9,0.999]$ from Eq. \ref{eq:smoothing} on a variety of datasets. We found that although performance is robust to the choice of $\beta_3$, a value of $\beta_3=0.99$ is the best general choice. Larger $\beta_3$'s result in slower adaptation of the step size to the loss landscape, but less susceptibility to noise. 

We performed Ablation studies on the impact of different $\beta_3$ and $c$ values. While we found $\beta_3$ to have a minor impact on performance different $c$ values do tend to have a more substantial impact. But in contrast to the learning rate, the $c$ hyperparameter seems to have an architecture and dataset independent optimum of $c=0.3$.

%% file: experiments.tex
In this section, we detail our experimental design to investigate the performance of our proposed optimization method.
We utilize datasets, model implementations and weights from the Huggingface library \cite{https://doi.org/10.48550/arxiv.1910.03771}, the pytorch datasets library and the nanoGPT \cite{nanoGPT} github repository. 

\subsection{Candidates} 
\label{sec:candidates}
A quick overview of all candidates we are evaluating can be seen below:

\begin{itemize}
    \item SGD  with tuned learning rate and learning rate schedule
    \item ADAM with tuned learning rate and learning rate schedule
    \item SGD + SLS, see Section \ref{sec:sgdarmijo} 
    \item ADAM + SLS, see Section \ref{sec:adamarmijo} 
    \item SGD + SaLSa, see Section \ref{sec:mini-batch}
    \item ADAM + SaLSa, see Section \ref{sec:mini-batch}
    \end{itemize}
    
As a baseline comparison we evaluate the ADAM and SGD optimizers with a cosine decay with warm starting for 10\% of the total training time. For NLP tasks this warm starting and cosine decay is common practice. For the image tasks we compare to a flat learning rate as done in \cite{vaswani20a}. 

We take the peak learning rate for ADAM on natural language tasks $\eta = 2 \cdot 10^{-5}$ from the original Bert paper by \cite{bert}, which presents a good value for a variety of classification tasks, including the Glue \cite{wang-etal-2018-glue} tasks upon which we are evaluating.
We found the value for the peak learning rate for SGD on the NLP task $\eta = 2 \cdot 10^{-3}$ using a grid search.

We found the value $\eta = 1 \cdot 10^{-3}$ for image classification for ADAM using a grid search. The same procedure resulted in $\eta = 1 \cdot 10^{-1}$ for SGD.

\subsection{Datasets and Models}
To evaluate an optimization method it is necessary to perform large scale runs of complex real world datasets and tasks. This is especially important as many optimization methods perform well on small scale or clean theoretical tasks, but fail to perform well on real world data.

\emph{Natural Language Processing - Transformers}

We consider a common scenario in natural language processing, where a large pre-trained language model (in our case Bert \cite{bert}) is fine-tuned on a small to medium sized dataset, similar to previous papers \cite{ijcnn2023}.
The Glue dataset, introduced by Wang et al. (2018) \cite{wang-etal-2018-glue}, comprises a diverse set of widely recognized natural language processing (NLP) classification tasks. This dataset serves as a benchmark for evaluating common NLP capabilities. The versions of the datasets utilized in this study are sourced from tensorflow-datasets 4.0.1.

To be specific, of the Glue collection \cite{wang-etal-2018-glue}, we use the datasets Stanford Sentiment Treebank \emph{SST2},
 Microsoft Research Paraphrase Corpus \emph{MRPC}, Stanford Question Answering Dataset \emph{QNLI},
 and the Multi-Genre Natural Language Inference Corpus 
  \emph{MNLI}. These datasets range from 500 - 400.000 training samples and represent a variety of fine-tuning tasks. 

  As a further evaluation metric for language models, we fine-tune GPT-2 \cite{gpt} on the Shakespeare dataset as described in \cite{nanoGPT}  implementation.

\emph{Image Classification - Convolutional Neural Networks}

In image classification common evaluation datasets are CIFAR10 and CIFAR100 \cite{cifar}, both being small scale (50.000 samples, 32x32 resolution). To obtain more reliable results we also compare on ImageNet \cite{imagenet} which consists of roughly $10^6$ samples.
We use the ResNet34 \cite{resnet} architecture without pre-training for small datasets and ResNet50 for ImageNet. 


\subsection{Implementation Details} 
The following details are the same for all experiments:
All models are trained 5 times and the averaged metrics are reported in Tables \ref{Fig:acc} and \ref{Fig:loss}. The learning curves as well as standard errors are visualized in Figures \ref{fig:nlp} and \ref{fig:image}. 

A Bert \cite{bert} model was trained on the NLP dataset with the following hyperparameter choices: Five epochs training time on each dataset. The pooling operation used in the Glue experiments is [CLS]. The maximum sequence length is set to 256 tokens. The batch size used during training is $32$.

For the image datasets CIFAR10 and CIFAR100 \cite{cifar} ResNet34 \cite{resnet} was used. For the ImageNet \cite{imagenet} dataset the ResNet50 \cite{resnet} architecture was used, a larger architecture was used due to the increased amount of complexity and size of the dataset. The batch size used during training is set to 128 for Cifar10 and Cifar100 and 256 for ImageNet. We applied pre-processing as described in the ResNet paper by \cite{resnet}. Models were trained on CIFAR10 and CIFAR100 for 100 epochs and on ImageNet for 12 epochs.


The computing time for all experiments was roughly 65 days on an A40 GPU. Roughly 15 of these days were used for the NLP tasks and 50 for the image datasets.

%% file: results.tex
In this section, we will describe the results of our experiments. 
We compare the 6 candidates as described in Section \ref{sec:candidates}. All metrics reported are average values obtained using 5 training runs.
All accuracies presented below are computed based on the validation sets. The losses presented are computed on the training sets and are smoothed using an exponential moving average. The shaded regions surrounding each line denote the standard error for each experiment. Figures \ref{fig:nlp} and \ref{fig:image} showcase the accuracies and losses throughout the training period. Tables \ref{Fig:acc} and \ref{Fig:loss} provide a representation of the peak accuracies and final losses for each candidate. 
In summary, we observe the SaLSa methods having on average an 1.5\% advantage on accuracy and a 50\% lower average log loss at end of training. 

\subsection{Natural Language Processing - Transformer Experiments}


\begin{figure*}[h]
\subfloat[MNLI ]{\includegraphics[width = 0.33\textwidth]{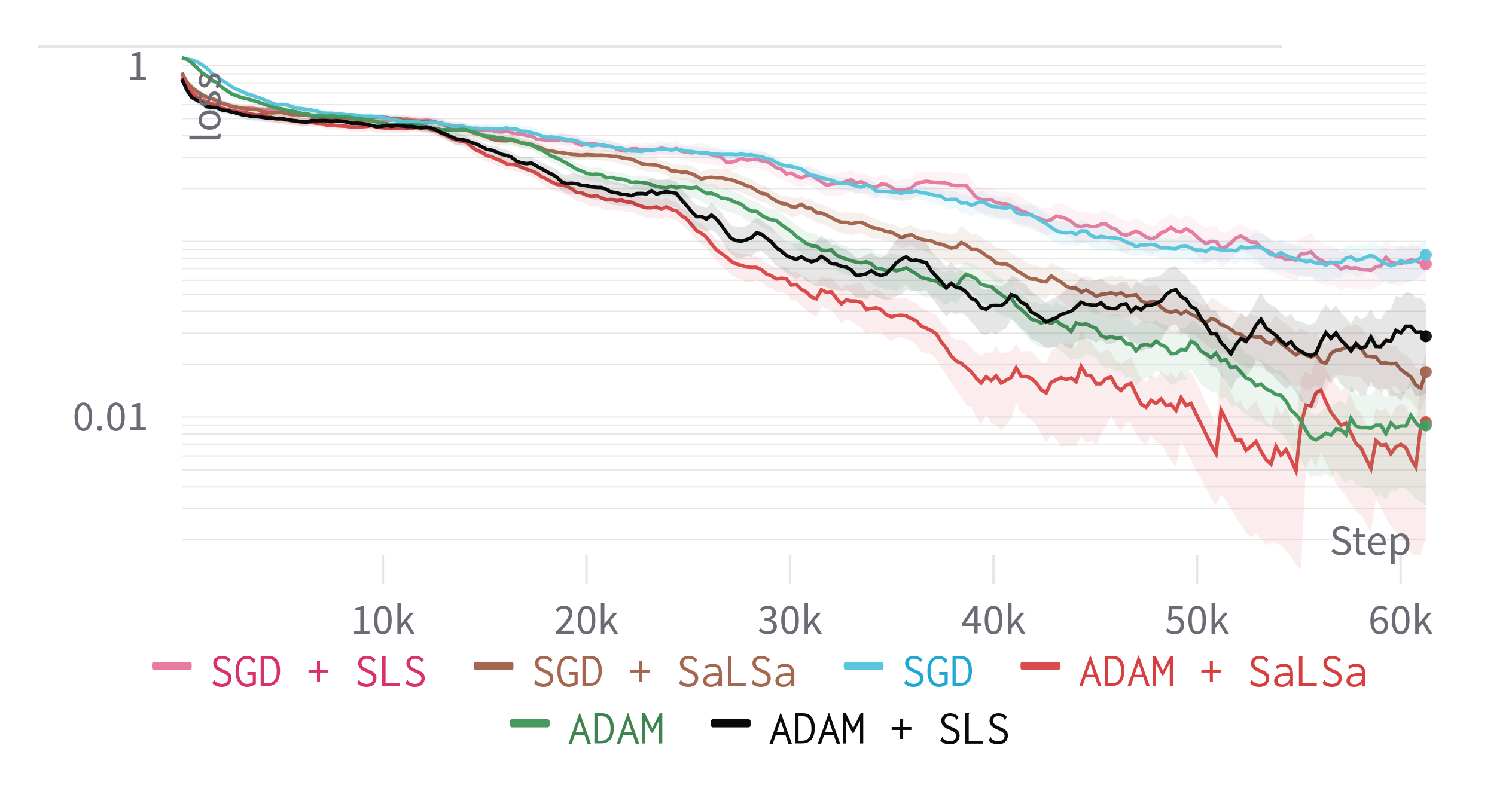}} 
\subfloat[MRPC ]{\includegraphics[width = 0.33\textwidth]{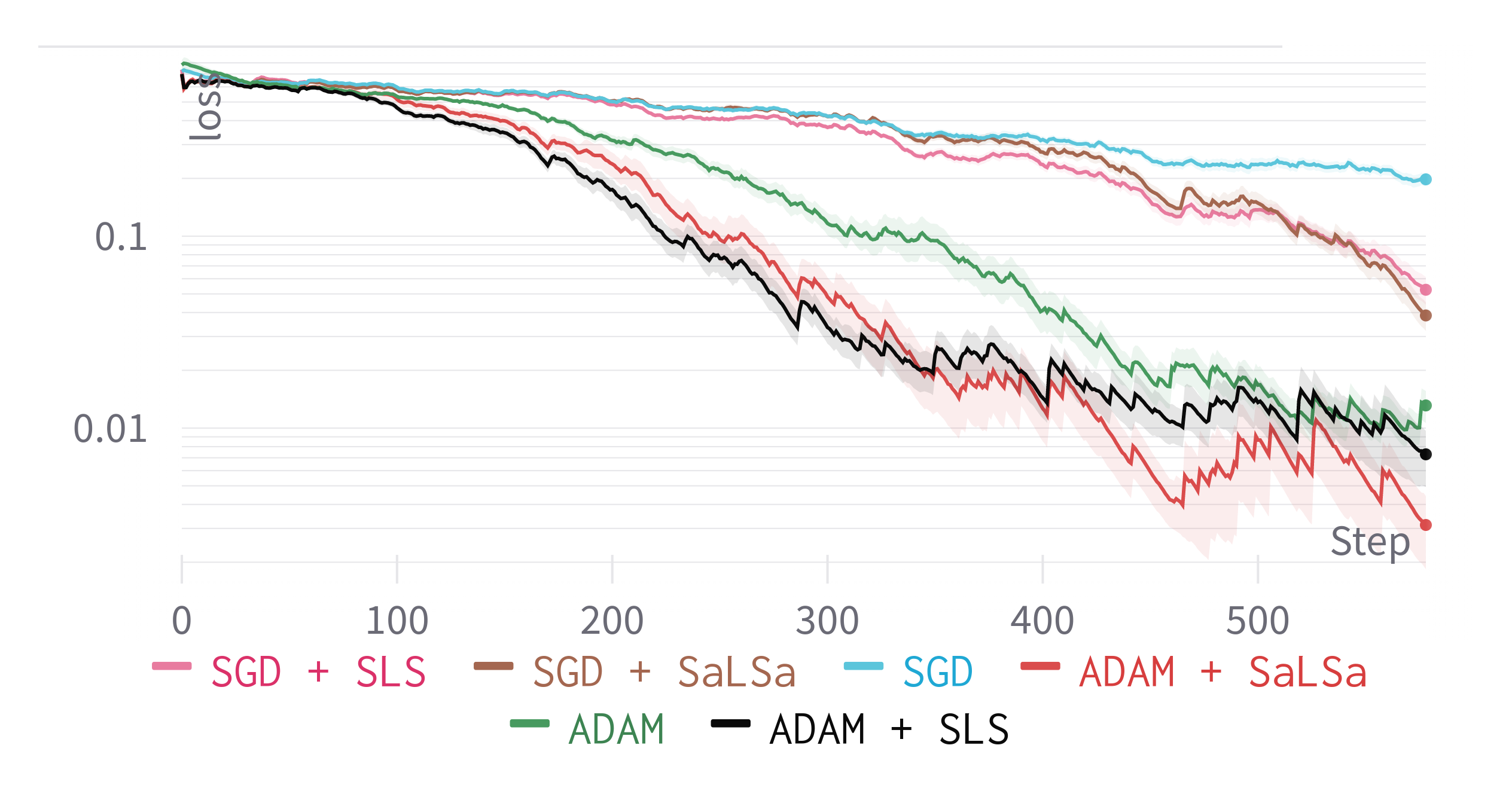}}
\subfloat[QNLI ]{\includegraphics[width = 0.33\textwidth]{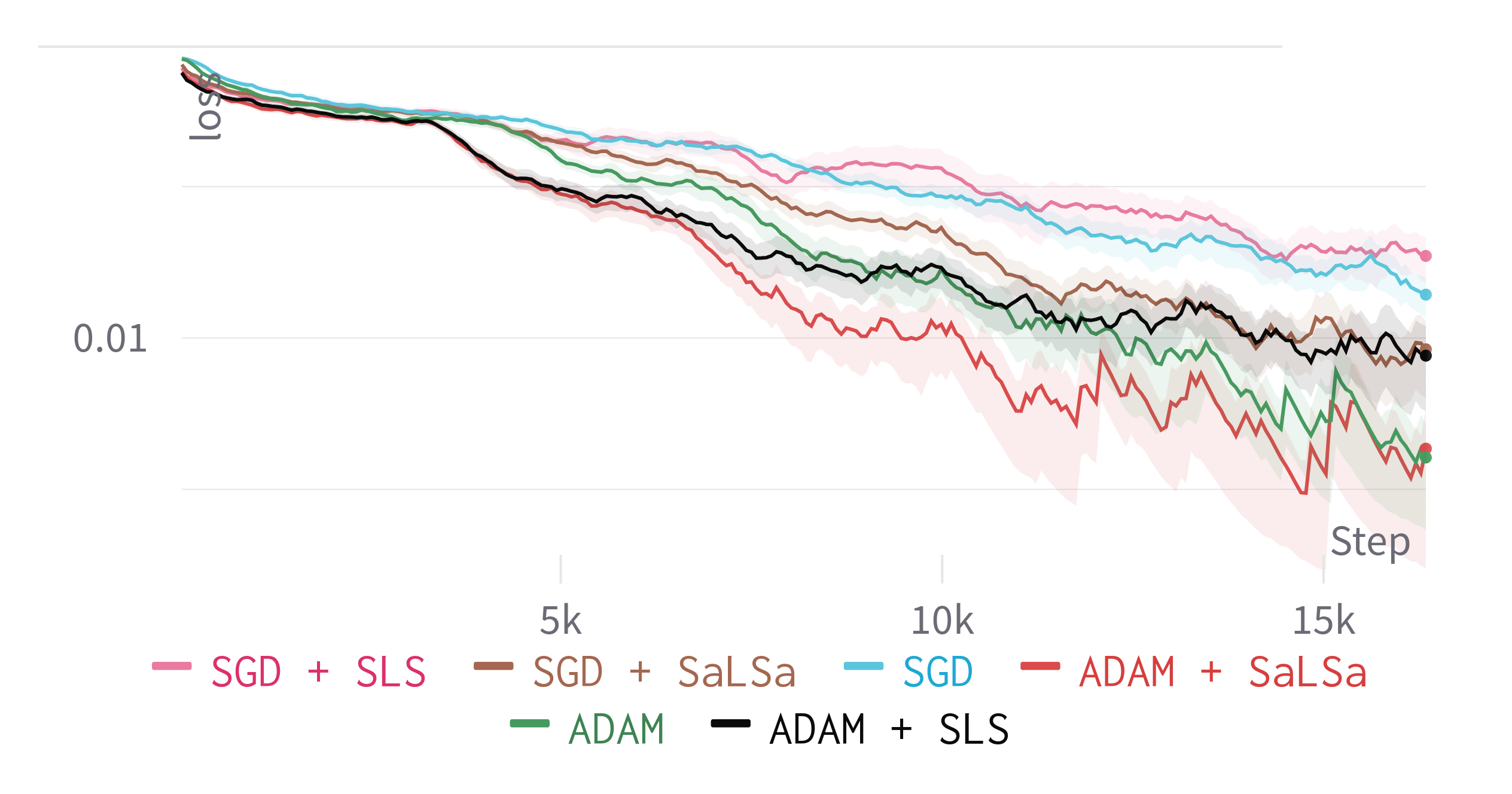}}
 \\
\subfloat[MNLI]{\includegraphics[width = 0.33\textwidth]{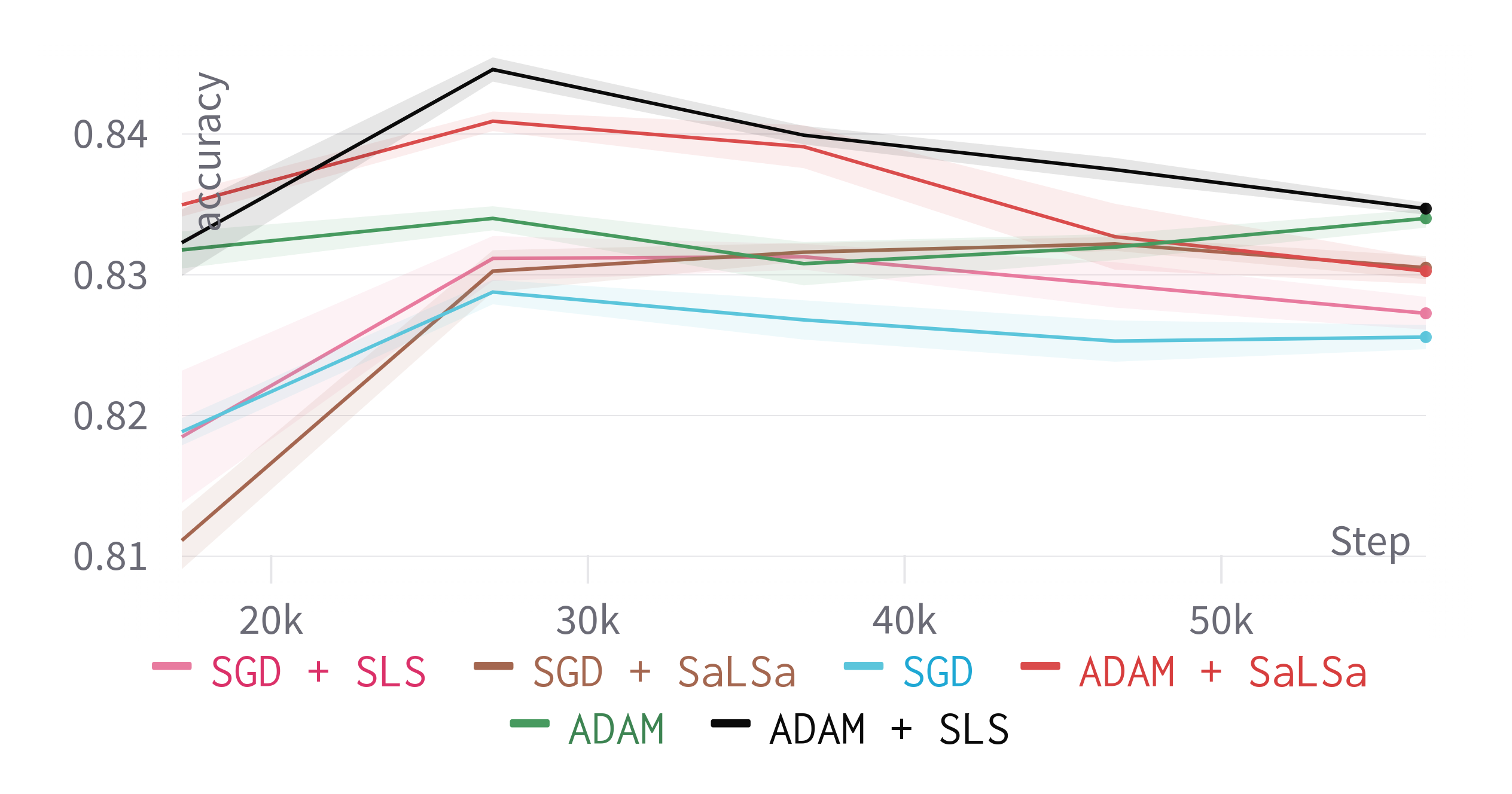}} 
\subfloat[MRPC]{\includegraphics[width = 0.33\textwidth]{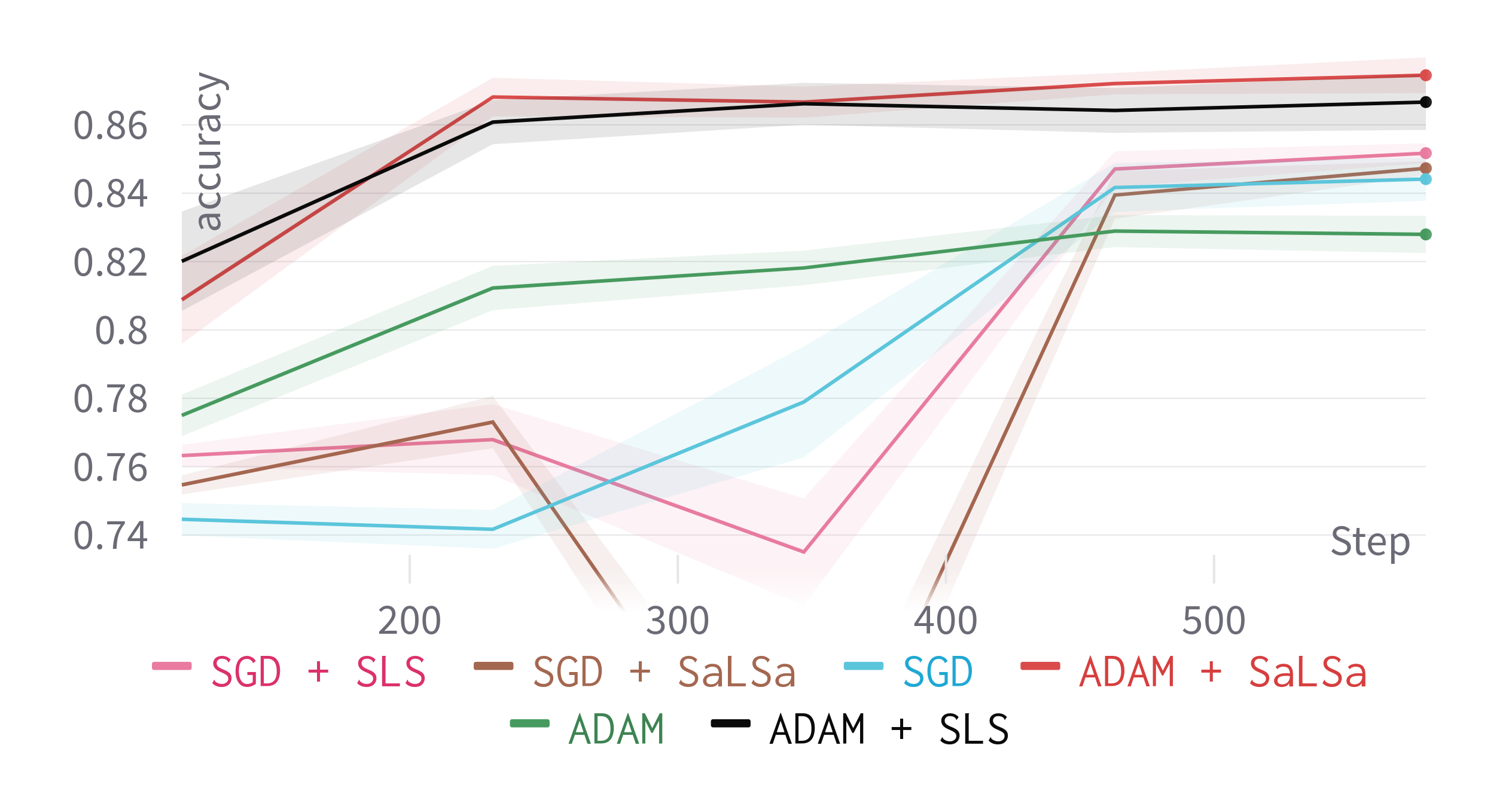}}
\subfloat[QNLI]{\includegraphics[width = 0.33\textwidth]{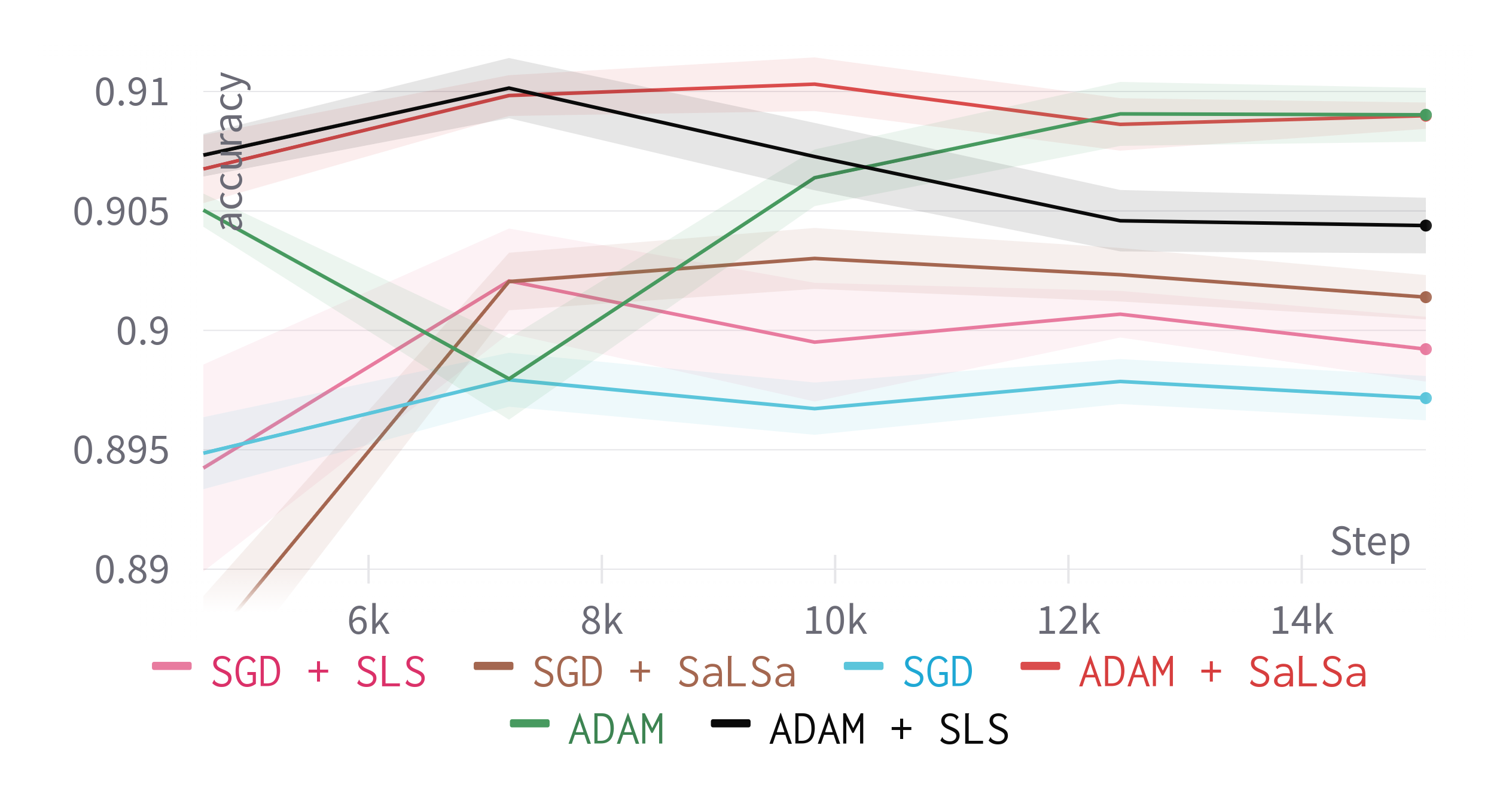}}

\caption{The top row depicts the loss curves, while the bottom row depicts the accuracy curves from experiments conducted on the GLUE dataset. Standard errors are represented around each curve with a shaded area. Accuracy measurements are calculated on the validation data, loss calculations were based on the training data. }
\label{fig:nlp}
\end{figure*}

\begin{figure*}
\vspace{-0.05\textwidth}
\subfloat[ImageNet]{\includegraphics[width = 0.33\textwidth]{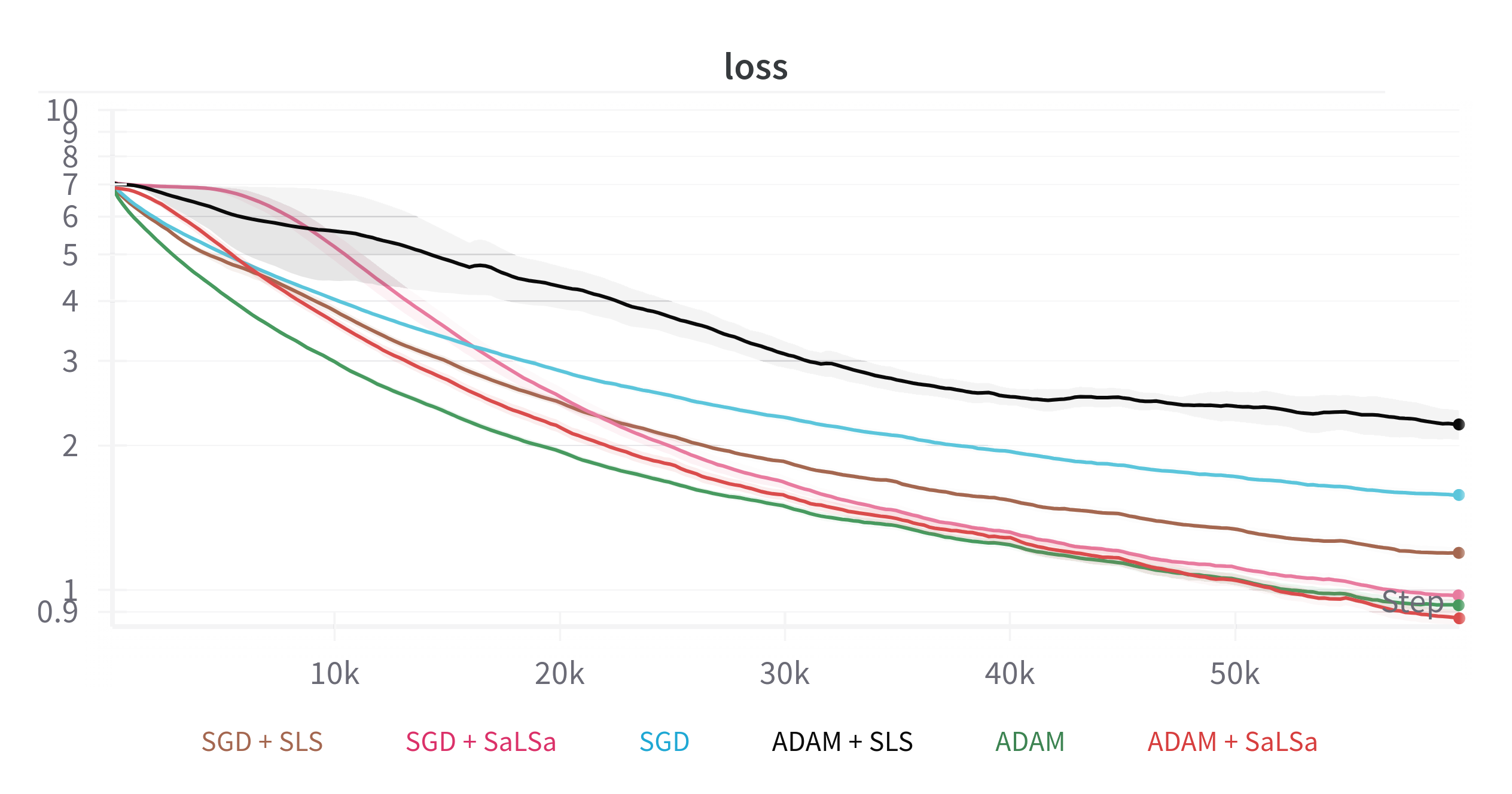}}
\subfloat[Cifar 100]{\includegraphics[width = 0.33\textwidth]{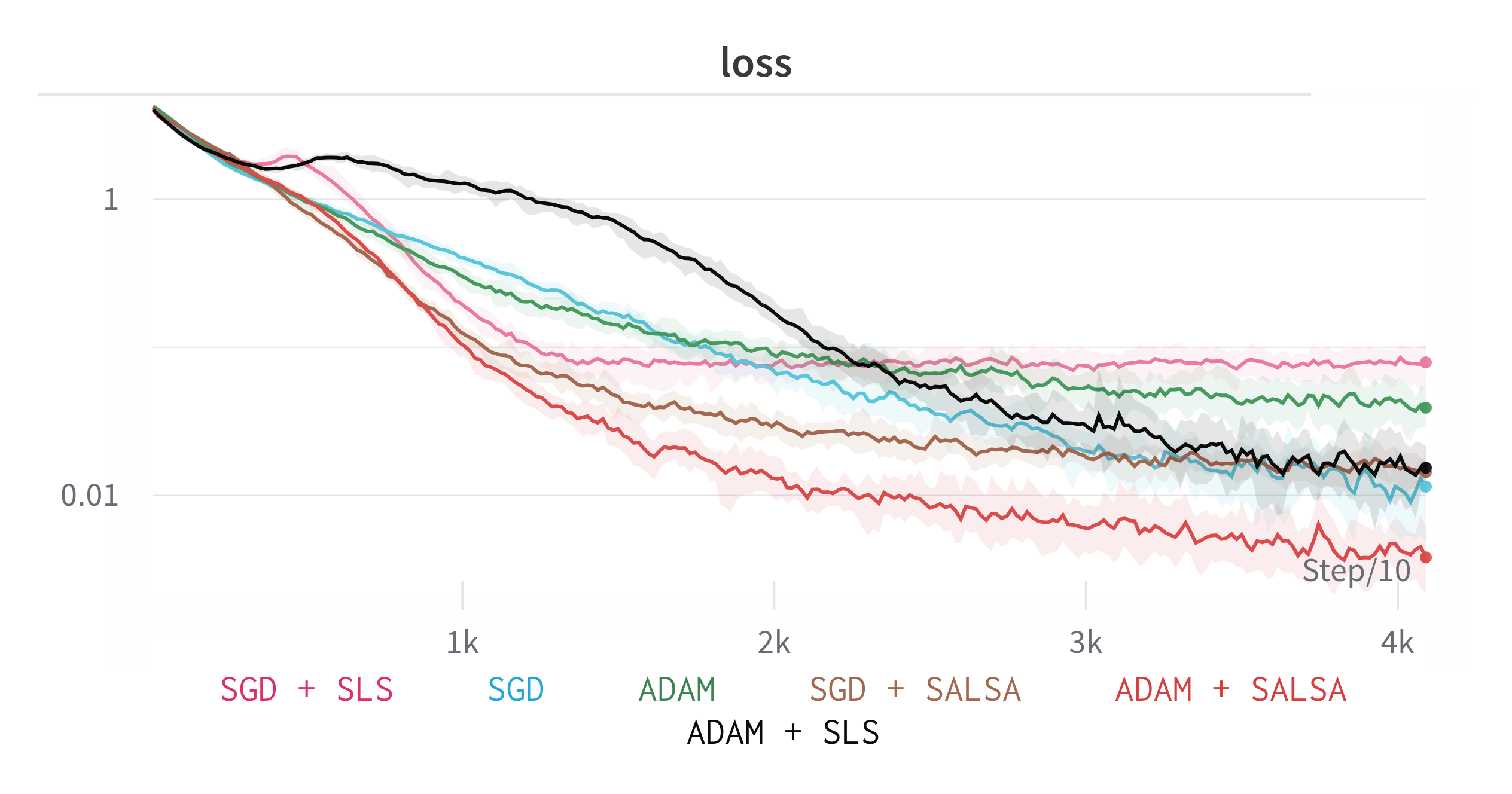}}
\subfloat[Cifar 10]{\includegraphics[width = 0.33\textwidth]{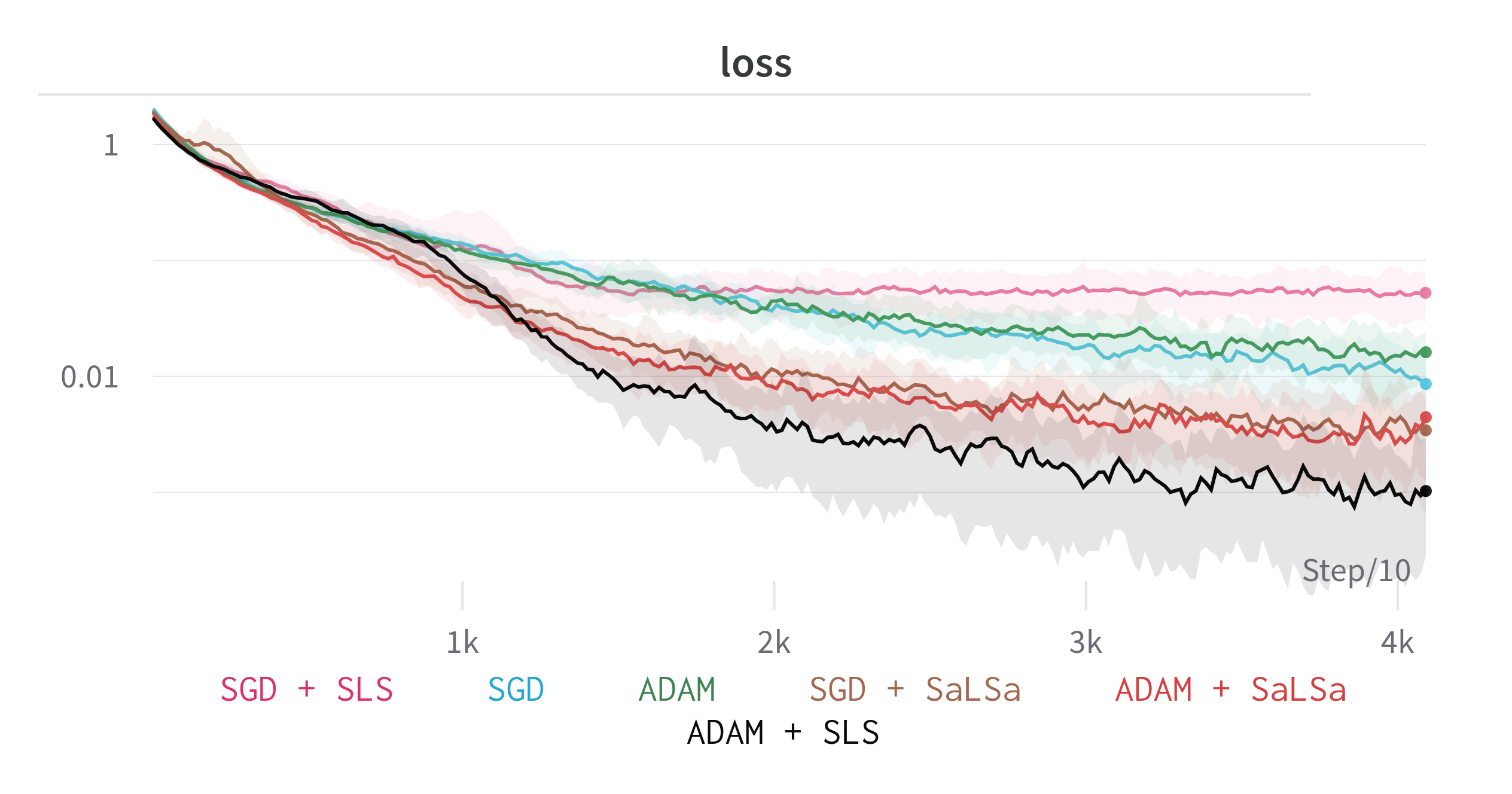}} \\
\subfloat[ImageNet]{\includegraphics[width = 0.33\textwidth]{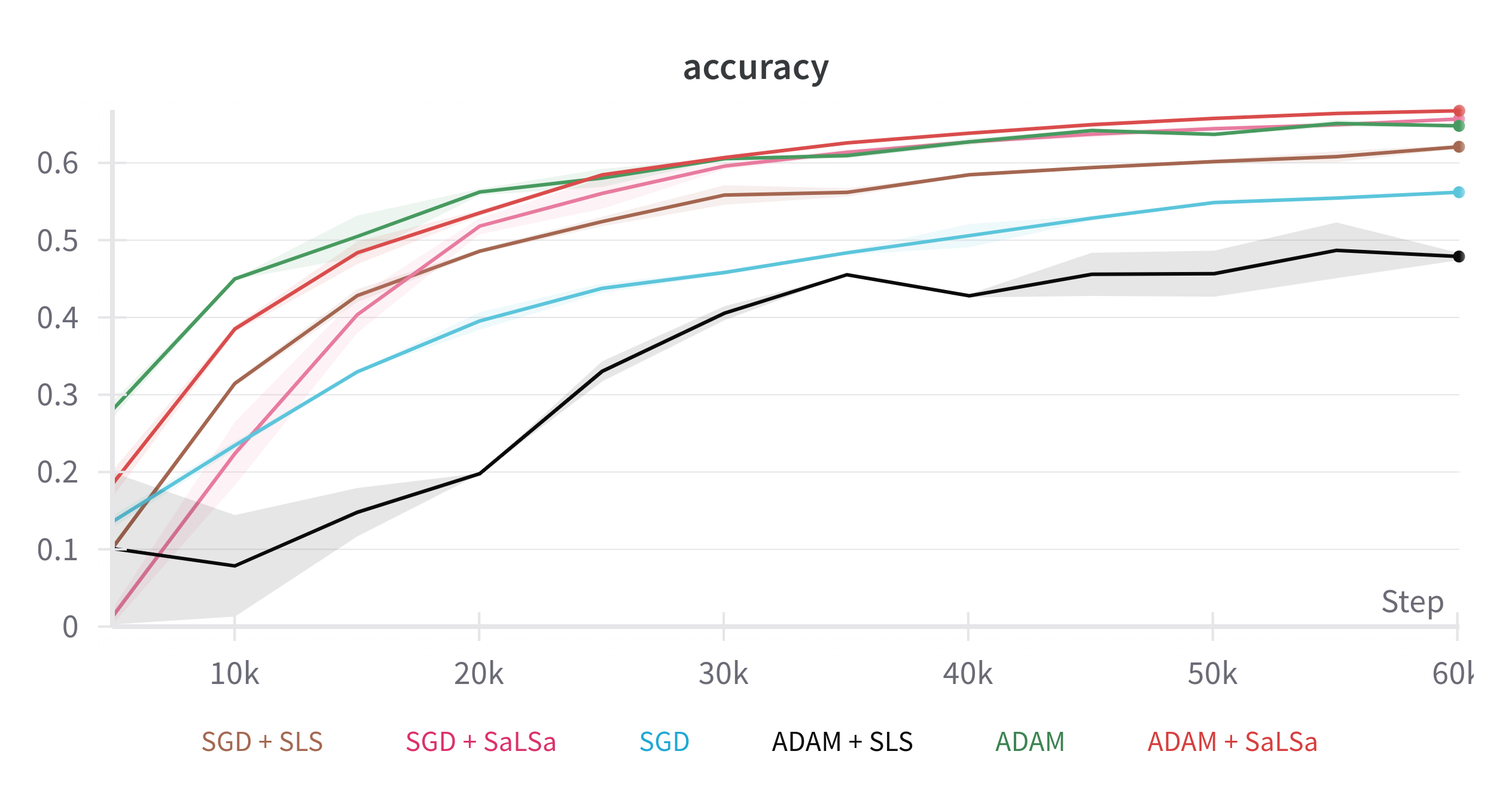}} 
\subfloat[Cifar 100]{\includegraphics[width = 0.33\textwidth]{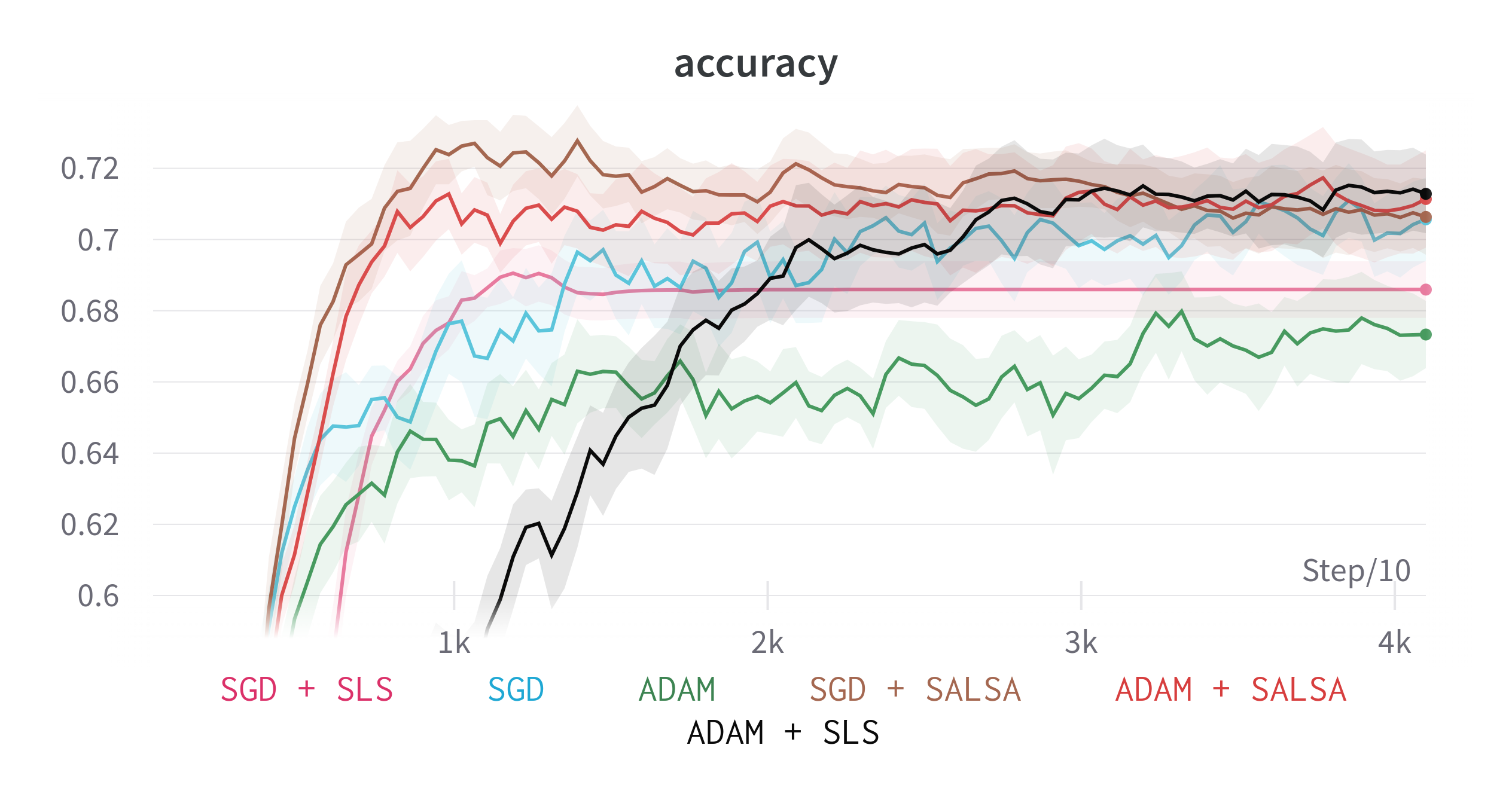}} 
\subfloat[Cifar 10]{\includegraphics[width = 0.33\textwidth]{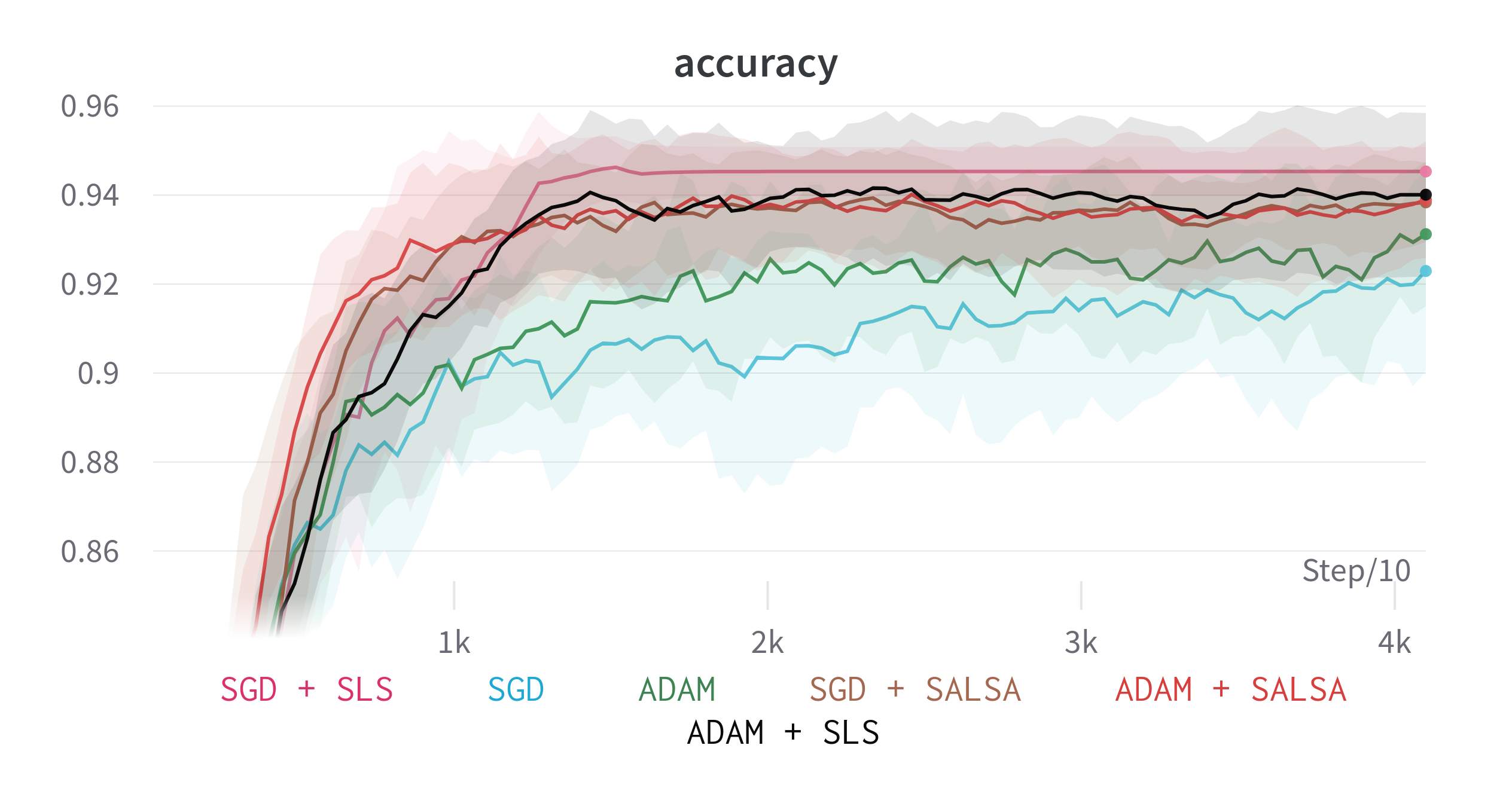}}

\caption{The top row depicts the loss curves, while the bottom row depicts the accuracy curves from experiments conducted on the image datasets. Standard errors are represented around each curve with a shaded area. Accuracy measurements are calculated on the validation data, loss calculations were based on the training data.}
\label{fig:image}
\vspace{-0.01\textwidth}
\end{figure*}

In our NLP experiments, as shown in Figure \ref{fig:nlp} and in the Appendix for GPT-2 and SST2, we have observed that, on average, ADAM + SaLSa achieves a lower final loss compared to ADAM, ADAM + SLS, and SGD + SLS. However, this improvement in loss does not always translate to a significant difference in the accuracy metric. ADAM + SLS and ADAM + SaLSa perform similarly in terms of accuracy, but both outperform ADAM and SGD + SLS on average.
Note that for similar final losses, the convergence rate of SaLSa is generally faster than that of ADAM, as depicted in Figure \ref{fig:nlp}.

\begin{table*}[t] 
  \centering
  \caption{Peak classification accuracies, averaged over 5 runs, for all datasets and optimization methods. Best performing optimization method is marked in \textbf{bold}. }
  \label{Fig:acc}
  \begin{tabular}{c cccccc}
    \toprule
 & ADAM & SGD & ADAM  & SGD  & ADAM  & SGD  \\
 & - & - & SLS & SLS &  SaLSa &  SaLSa \\
 \cmidrule(r){1-1}   \cmidrule(r){2-7} 
\it{MNLI} & 0.8340 & 0.8256 & \textbf{0.8446}  & 0.8303 & 0.8409 & 0.8305    \\
\it{QNLI} & \textbf{0.9090} & 0.8972 & 0.9044  & 0.8971 & \textbf{0.9090} & 0.9014   \\
\it{MRPC} & 0.8279 & 0.8441 & 0.8667  & 0.8441 & \textbf{0.8745} & 0.8473   \\
\it{SST2} & \textbf{0.9271} & 0.9225 & 0.9261   & 0.9167 & 0.9216  & 0.9245   \\
 \cmidrule(r){1-1}   \cmidrule(r){2-7} 

 ResNet34 \\
\it{CIFAR10} & 0.9312 & 0.9229 & 0.9401 & \textbf{0.9453}  & 0.9389 & 0.9384   \\
\it{CIFAR100} & 0.6733 & 0.7057 & \textbf{0.7128} & 0.6859 & 0.7114  & 0.7064   \\
 ResNet50 \\
\it{ImageNet} & 0.6486 & 0.5633 & 0.4836 & 0.6234 & \textbf{0.6684} & 0.6594   \\
 \cmidrule(r){1-1}   \cmidrule(r){2-7} 

average & 0.8215 & 0.8116 & 0.8111 & 0.8204 & \textbf{0.8378} & 0.8297    \\
average rank & 3.57 & 4.86 & 2.42 & 4.43 & \textbf{2.14} & 3.28    \\
    \bottomrule
  \end{tabular}
\end{table*}

\begin{table*}
  \centering
  \caption{Final losses, averaged over 5 runs, for all datasets and optimization methods. Best performing (minimal loss) optimization method is marked in \textbf{bold}. The logarithmic average is taken due to the logarithmic nature of the typical loss.}
  \label{Fig:loss}
  \begin{tabular}{c cccccc}
    \toprule
 & ADAM & SGD & ADAM  & SGD  & ADAM  & SGD  \\
 & - & - & SLS &  SLS &  SaLSa &  SaLSa \\
 \cmidrule(r){1-1}   \cmidrule(r){2-7} 
\it{MNLI} & 0.009567 & 0.08613 & 0.03713  & 0.06901 & \textbf{0.005867} &  0.02174  \\
\it{QNLI} & 0.00258 & 0.02079 & 0.00504  & 0.03667 & \textbf{0.000628} & 0.0091627\\
\it{MRPC} & 0.01312 & 0.1978 & 0.007298  & 0.05262 & \textbf{0.003126}  & 0.03862  \\
\it{SST2} & \textbf{0.005857} & 0.02561 & 0.009457   & 0.0412 & 0.006991 & 0.01837  \\
GPT-2 & 2.86 & 3.572 & 2.917  & 3.566 & \textbf{2.772} & 3.559  \\
 \cmidrule(r){1-1}   \cmidrule(r){2-7} 
ResNet34\\
\it{CIFAR10} & 0.01394 & 0.00982 & \textbf{0.0009508} & 0.05646 & 0.003314  & 0.003773  \\
\it{CIFAR100} & 0.03739 & 0.01143 & 0.01337  & 0.08245 & \textbf{0.003774} & 0.01453  \\
ResNet50\\
\it{ImageNet} & 0.9122 & 1.547 & 2.036 & 1.144 & \textbf{0.8339}  &  0.9788 \\
 \cmidrule(r){1-1}   \cmidrule(r){2-7} 

log average & 0.0355 & 0.0930 & 0.0315  & 0.134 & \textbf{0.0148} & 0.0477 \\
average rank & 2.75 & 4.625 & 3.125  & 5.5 & \textbf{1.25} & 3.75 \\
    \bottomrule
  \end{tabular}
\end{table*}

\subsection{Image - Convolutional Neural Networks Experiments}

In our image experiments, we have observed that the combination of ADAM + SLS or SGD + SLS yields good results for CIFAR10 and CIFAR100, but performs poorly for ImageNet, as depicted in Figure \ref{fig:image}. We attribute this outcome primarily to stability issues. Specifically, ADAM + SLS occasionally produces excessively large step sizes $\eta$, or it diminishes them to unreasonably small values $\eta \leq 10^{-10}$. On the other hand, our enhanced approaches ADAM + SaLSa and SGD + SaLSa, do not encounter these problems and on average deliver the best performance among all methods. 




%% file: related_work.tex

The refinement of deep neural network optimization has been a focal point of investigation within the realm of machine learning. Various techniques and optimizers have been proposed, including but not limited to SGD \cite{robbins51a}, Adagrad \cite{adagrad}, RADAM \cite{DBLP:journals/corr/abs-1908-03265}, ADAMW \cite{adamw}, RMSprop \cite{RMSprop} and Adam \cite{adam}. However, selecting the most suitable optimizer remains a challenge, and there is no clear consensus on the best according to \cite{schmidt2021descending}.

In a recent study by\cite{schmidt2021descending} on the topic of optimization methods, it was observed that while there are various optimizers available, there is no definitive best optimizer. The authors highlight that the introduction of more optimizers does not necessarily lead to improved results, and therefore, alternative approaches should be explored to enhance optimization techniques. One such approach with great potential is automatic step size determination. One of the most common approaches for this are line search methods, which hold promise for enhancing optimization processes \cite{DBLP:journals/corr/abs-1805-08890, vaswani20a, paquette20a, galli2023dont}.


In this work we particularly build upon \cite{vaswani20a}. 
The Armijo line search method introduced there, offers several important advantages over other optimization techniques:
no hyperparameter tuning of the learning rate, faster convergence rates and better generalization.
A significant drawback of this method, along with other line search approaches, is that it requires at least an additional forward pass per update step. Consequently, this leads to an increase of approximately 30\% in computational resources required per training step. 

Recent work has shown that Transformers are highly sensitive to the choice of learning rate and learning rate schedule schedule during training \cite{DBLP:journals/corr/abs-1908-03265, idealme}. To address this issue, various approaches have been proposed, such as RADAM \cite{DBLP:journals/corr/abs-1908-03265} and warm starting. We have showed previously that our methods are able to train Transfomers despite these challenges.
Other related work includes \cite{batchnoise} which studies the correlation between batch size and learning rate, \cite{googleautodiff} which shows theoretical and practical results for training using higher order gradients and \cite{chen22a} which investigates why Adam, upon which we focus, is so effective at training the Transformer architecture.

The optimization of neural networks continues to be an important area of research, to which, the development of effective and reliable line search methods, which work on sensitive architectures such as transformers or large scale convolutional neural networks, constitutes a significant contribution.

%% file: conclusion.tex

We have introduced SaLSa, an automatic step size selection method and built a general purpose optimizer on top. We have compared its performance against tuned learning rates for larger datasets and architectures than previously done in optimizer evaluations for line search methods. The SaLSa optimizer performance compares favorably in these cases, while requiring no tuning of learning rates, or code overhead, as well as minimal compute overhead. We recommend its use as a first choice for training deep neural networks in these domains and publish the code as a Python package.

The source code is open-source and free (MIT licensed) software and available on \href{https://github.com/TheMody/No-learning-rates-needed-Introducing-SALSA-Stable-Armijo-Line-Search-Adaptation}{github}



